\documentclass[letterpaper, 10 pt, conference]{ieeeconf}
\usepackage{times}
\usepackage[pdftex]{graphicx}
\usepackage{subfigure}
\usepackage{amsmath,amssymb,amsopn,amstext,amsfonts}
\usepackage{cancel}
\usepackage[space]{cite}
\usepackage{pdfsync}
\usepackage{balance}
\usepackage{color}
\usepackage{mathtools}
\usepackage{makecell}
\usepackage{algorithm}
\usepackage{algorithmic}
\usepackage{bm}
\usepackage{diagbox}
\usepackage{float}
\usepackage{epstopdf}
\usepackage{pifont}
\usepackage{multirow}
\usepackage{url}
\usepackage{tabularx}
\usepackage{bm}
\usepackage{multirow}

\usepackage[table,xcdraw]{xcolor}
\usepackage[linkcolor=black,citecolor=black,urlcolor=black,colorlinks=true]{hyperref}

\usepackage{verbatim} 

\newtheorem{problem}{Problem}[section]
\newtheorem{lemma}{Lemma}[section]
\newtheorem{proposition}{Proposition}
\newtheorem{definition}{Definition}

\graphicspath{{../figure/}}
\DeclareGraphicsExtensions{.png,.jpg,.eps,.pdf}
\IEEEoverridecommandlockouts
\title{\LARGE \bf Certifiably Optimal Mutual Localization with \\
Anonymous Bearing Measurements }
\author{Yingjian Wang\textsuperscript{1,2}, Xiangyong Wen\textsuperscript{1,2}, Longji Yin\textsuperscript{2}, Chao Xu\textsuperscript{1,2}, Yanjun Cao\textsuperscript{2}, Fei Gao\textsuperscript{1,2}\thanks{\textsuperscript{1}State Key Laboratory of Industrial Control Technology, Institute of Cyber-Systems and Control, Zhejiang University, Hangzhou, 310027, China.}
	\thanks{\textsuperscript{2}Huzhou Institute of Zhejiang University, Huzhou, 313000, China.}
	\thanks{E-mails:\tt\small \{yj\_wang, fgaoaa\}@zju.edu.cn.}
	}

\begin{document}

\maketitle
\thispagestyle{empty}
\pagestyle{empty}

\begin{abstract}
	Mutual localization is essential for coordination
	and cooperation in multi-robot systems. Previous works have
	tackled this problem by assuming available correspondences
	between measurements and received odometry estimations,
	which are difficult to acquire, especially for unified robot
	teams. Furthermore, most local optimization methods ask
	for initial guesses and are sensitive to their quality. In this
	paper, we present a certifiably optimal algorithm that uses
	only anonymous bearing measurements to formulate a novel
	mixed-integer quadratically constrained quadratic problem
	(MIQCQP). Then, we relax the original nonconvex problem into
	a semidefinite programming (SDP) problem and obtain a certifiably
	global optimum using with off-the-shelf solvers. As a
	result, our method can determine bearing-pose correspondences
	and furthermore recover the initial relative poses between
	robots under a certain condition. We compare
	the performance with local optimization methods on extensive
	simulations under different noise levels to show our advantage
	in global optimality and robustness. Real-world experiments
	are conducted to show the practicality and robustness.
\end{abstract}

\section{Introduction}
\label{sec:introduction}
Recently, due to the inherent advantage, multi-robot systems have
received increasing attention in many applications, such as formation control\cite{quan2021distributed}, exploration\cite{gao2021meeting}, search and rescue and surveillance. To execute each subtask correctly and complete the full task collaboratively, robots in a team are expected to be located in a common reference frame. However, this requirement is not satisfied in wild environments like underground caves where global coordinate systems are not available. Launching robots in a predetermined relative pose is another solution. However, it is obviously time-consuming and prone to failure in large-scale environments. 

To bridge this gap, self-localization using onboard sensors and relative pose recovery are irreplaceable in multi-robot systems. There are majorly two ways to estimate the initial relative transformations between robots in a team. They are map-based localization which relies on exchanging environment features, and mutual localization which depends on robot-to-robot measurements. Most research focuses on the map-based relative pose recovery method, which can be easily adapted from loop-closing modules of existing simultaneous localization and mapping (SLAM) systems. However, it requires robots to observe the same scene and send observed environment information to others, leading to degeneration in the environments with many similar or texture-less scenes.

Our study focuses on mutual localization using bearing measurements, which only utilize detected robots' 2D coordinates in the observer's image and observed robots' estimated odometry. Compared to map-based localization, it is less influenced by environments and needs less bandwidth. Despite its appeals, as we do not rely on any specialized devices, like visual tags or external sensors, data association between the visual detection and robot identifications in a team of unified robots is challenging.
\begin{figure}[t]
	\centering
	\includegraphics[width=0.5\textwidth]{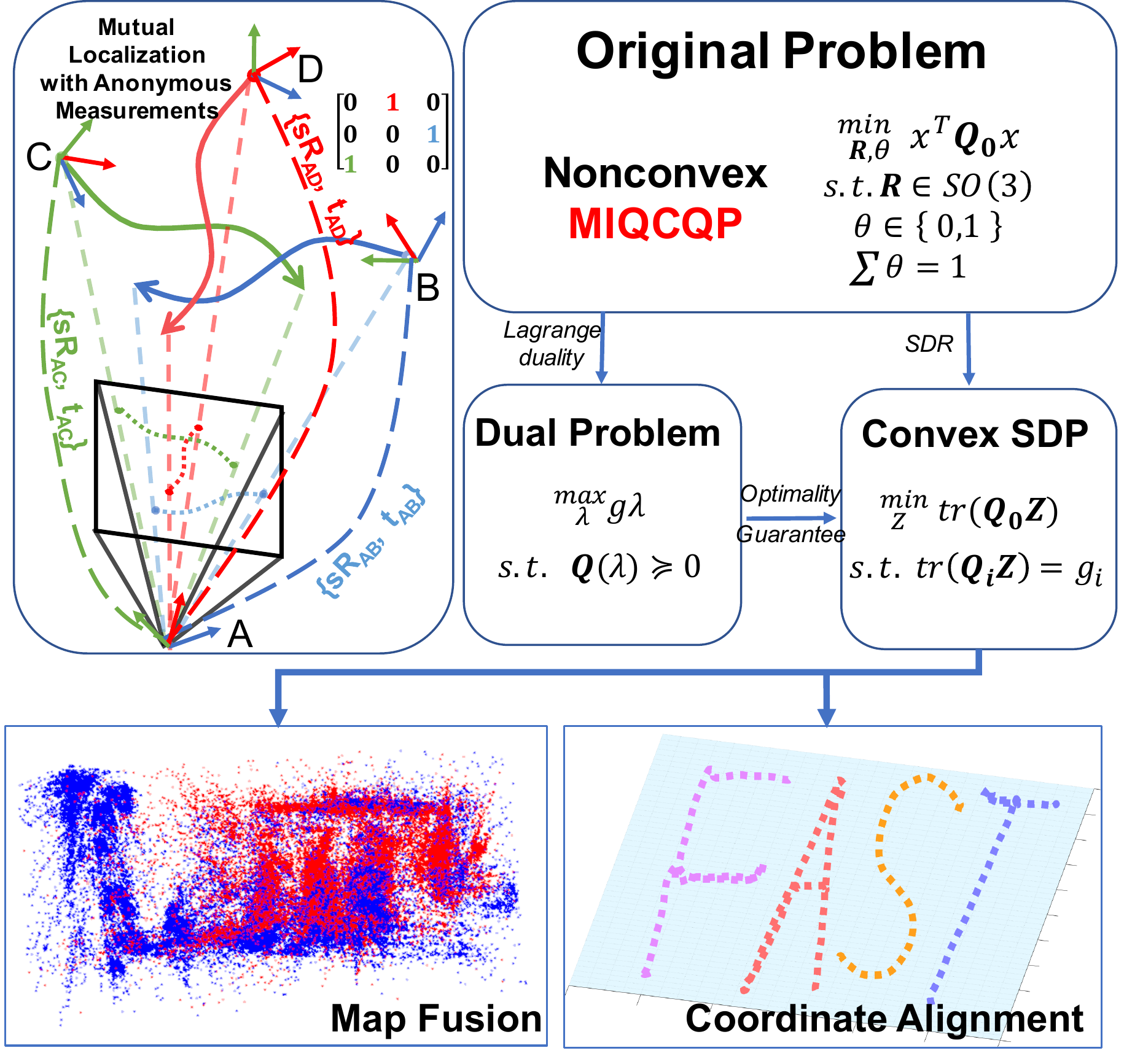}
	\caption{\label{fig:fast} Overview of our proposed method which can obtain certifiably optimal solution for mutual localization problem with anonymous bearing measurements. Our result can be used for map fusion in multi-robot monocular SLAM and coordinate alignment in multi-robot tasks.}
	\vspace{-1cm}
\end{figure}
For this problem, existing works take the similar paradigm of establishing data association firstly and then recovering the relative pose with extra sensors such as IMU. Distinctly, in our paper, we introduce binary variables representing the data association relationships and mix them with multiple SO(3) variables representing the relative poses between observer and observed robots, formulating a mixed-integer problem. Furthermore, we rewrite it as a non-convex MIQCQP problem and employ tight convex relaxation to obtain a SDP problem. Thanks to its convexity, we obtain a \emph{certifiably globally optimal solution} to our formulated problem. Moreover, we also provide a condition, under which our approach avoids local minima in noise-free cases. Complete algorithm is demonstrated in Fig.\ref{fig:fast}. Extensive experiments on synthetic real-world datasets show the robustness of our method under different levels of noise.

Our contributions in this paper are:
\begin{enumerate}
    \item We provide an innovative formulation which jointly solves data association and relative poses in a MIQCQP problem. To the best of our knowledge, there is no such work in mutual localization.
    \item We propose an algorithm for the non-convex MIQCQP problem, which adopts semidefinite relaxation (SDR) to make it convex. Furthermore, we provide a condition to guarantee the tightness of the relaxation. 
    \item We conduct sufficient simulation and real-world experiments to validate the practicality and robustness of our proposed method.
    \item We release the implementation of our method in MATLAB and C++ for the reference of our community.
\end{enumerate}

\section{Related works}
\label{sec:related_works}

\subsection{Relative Pose Estimation}
\label{mutual_localization}
There are mainly two ways to solve multi-robot relative pose estimation (RPE) problems: interloop detection based methods and mutual observation based methods. Most interloop detection based methods, including centralized \cite{riazuelo2014c2tam, schmuck2019ccm} and decentralized architectures \cite{cieslewski2018data, lajoie2020door}, firstly determine whether the robots in a team visited the same places using loop detection technique \cite{arandjelovic2016netvlad}, then conduct the relative pose recovery. However, interloop detection-based methods require significant computation and bandwidth and have poor performance in environments with many similar scenes. 

Most mutual observation-based methods employ robot-to-robot range or bearing measurements to recover relative poses. Early work \cite{zhou2006multi, martinelli2005multi, chang2011vision} take extended Kalman filter (EKF) as nonlinear estimator using prior identified range measurements. Zhou \cite{zhou2012determining} provides a set of 14 minimal analytical solutions that cover any combination of range and bearing measurements. However, their proposed algorithm has poor performance under noise because it only uses minimal measurements. Besides, all the above works assume that correspondence between measurement and estimated poses is known, which is not common in practical applications.

Cognetti \cite{cognetti20123} and Franchi \cite{franchi2013mutual} solve mutual localization problem with particle filters (PF) using anonymous measurements. Indelman \cite{indelman2014multi} and Dong \cite{dong2015distributed} formulate a multi-robot pose graph problem and utilize the expectation-maximization (EM) approach to estimate initial relative poses between robots. However, it is well known that PF and EM all require extensive computation. Nguyen \cite{nguyen2020vision} adapts the coupled probabilistic data association filter to estimate relative pose with vision sensor and IMU. In \cite{jang2021multirobot}, Jang proposes an alternating minimization algorithm to optimize relative poses in multi-robot monocular SLAM. However, these local optimization methods are sensitive to initial values and cannot work with multiple bearing measurements in one image. Compared with the above work, our proposed method solves correspondence and relative poses together without extra sensor inputs.

\subsection{Certifiably Global Optimization}
\label{optimization}

Recently, based on semidefinite relaxation and advanced optimization theory, the research community has developed certifiably optimal non-minimal solvers for many computer vision and robotics problems that are non-convex and NP-hard. In \cite{carlone2015lagrangian}, Carlone uses Lagrangian duality to verify the optimality of candidate solution of pose graph optimization (PGO). Exploiting the strong duality of PGO, SE-Sync \cite{rosen2019se} and Cartan-Sync \cite{briales2017cartan} obtain the optimal solution of PGO under acceptable noise. In \cite{yang2019quaternion}, point registration with outliers is formulated as a QCQP by binary cloning, relaxed using SDR, and finally globally optimized by adding redundant constraints. Besides, SDR is also leveraged in 3d registration \cite{briales2017convex}, camera pose estimation \cite{briales2018certifiably, zhao2020efficient}, extrinsic calibration \cite{giamou2019certifiably} and so on. All of these problems involve optimization over SO(3) or SE(3) variables and add orthogonality constraints to make the convex relaxation tight. In this paper, our solution procedure is similar to \cite{yang2019quaternion}. Differently, we keep binary variables and introduce \emph{binary constraint} and \emph{correspondence constraint} to formulate a MIQCQP problem. As far as we know, our proposed algorithm is the first method that can obtain a globally optimal solution for the mutual localization problem using anonymous measurements.

\section{Formulation of Relative Pose Estimation}
\label{sec:problem formulation}

In this section, we formulate the RPE problem with anonymous measurements as a QCQP problem. Firstly, we define a loop error for mutual localization of one observed robot case in Sec.\ref{subsec:loop constraint}. Then in Sec.\ref{subsec:mixed-binary}, we extend the error to multiple observed robots case, introduce binary variables for data association, and formulate the optimization as a mixed-integer programming problem. Finally, we marginalize distance variables, define auxiliary variables, and derive a QCQP problemin Sec.\ref{subsec:marginalization}.

\subsection{Loop Error for One Observed Robot}
\label{subsec:loop constraint}
In this subsection, we consider two robots, observer robot $A$ and observed robot $B$, moving along two 3D trajectories. Their camera coordinates frame at time $j$ are denoted by $\{A_j\}$ and $\{B_j\}$, where $j \in J$, $J$ is the timestamp collection. Robot $A$ observes feature of robot $B$ at time $j$ and gets the bearing measurement ${b^B_{j}}$ in frame \{$A_j$\}. Assuming $B$ be rigid body, the inner bias $^{B}P$ between the feature and camera on $B$ are time-invariant, i.e., $^{B}P=^{B_j}P$. Then $^{A_j}P$, the feature coordinate in frame $\{A_j\}$, can be given by

\begin{equation}
^{A_j}P = D^B_j b^B_j = R_{A_jB_j} {^{B_j}P} + t_{A_jB_j} = R_{A_jB_j} {^BP} + t_{A_jB_j},
\end{equation}
where $D^B_j$ is the distance between $A$'s camera and the observed feature. $R_{A_jB_j}$ and $t_{A_jB_j}$ denote the relative rotation and translation between $\{A_j\}$ and $\{B_j\}$. For simplicity, we set $D_j = D^B_j$ and $b_j = b^B_j$ in two robots' case. And for each time $j$, we have
\begin{equation}
\begin{aligned}
\label{equ:instance}
R_{A_1A_j} {^{A_j}P} + t_{A_1A_j} = s_{AB}R_{AB} (R_{B_1B_j} {^BP} + t_{B_1B_j}) + t_{AB},
\end{aligned}
\end{equation}
where $s_{AB}$ denotes the scale ratio between local maps of $A$ and $B$, and \{$s_{AB}R_{AB},t_{AB}$\} is the corresponding relative pose. After subtraction between Equ. (\ref{equ:instance}) of $j_1,j_2\in J$, we eliminate variable $t_{AB}$ and derive the \emph{loop error}:
\begin{equation}
\begin{aligned}
\label{equ:constraint}
e^{AB}_{j_1j_2} = R_{A_1A_{j_2}} b_{j_2} D_{j_2} - R_{A_1A_{j_1}} b_{j_1} D_{j_1} + \widehat{t}_{A_{j_1}A_{j_2}} - \\
(R_{AB} \widehat{R}_{B_{j_1}B_{j_2}} {^B\bar{P}} + s_{AB}R_{AB} \widehat{t}_{B_{j_1}B_{j_2}}),
\end{aligned}
\end{equation}
where $\widehat{t}_{X_{j_1}X_{j_2}} = t_{X_1X_{j_2}}-t_{X_1X_{j_1}}, X\in \{A,B\}$ , $\widehat{R}_{B_{j_1}B_{j_2}} = R_{B_1B_{j_2}}-R_{B_1B_{j_1}}$ and $^B\bar{P}=^BP /s_{AB}$. If $s_{AB}R_{AB}$ and $^B\bar{P}$ are recovered, $^BP$ can be determined solely. This expression is found in \cite{jang2021multirobot}. In this paper, we reformulate it in a linear expression which will be used to get a quadratic cost in Sec.\ref{subsec:mixed-binary}. Firstly we define the following variables:
\begin{equation}
\begin{aligned}
\label{equ:definition}
r_s & \doteq \text{vec}(s_{AB}R_{AB}) \in \mathbb{R}^{9 \times 1}, \\
r_p & \doteq  \text{vec}(^B\bar{P}^T \otimes R_{AB}) \in \mathbb{R}^{27 \times 1}, 
\end{aligned}
\end{equation}
where $\otimes$ is the Kronecker product, $\text{vec}(M)$ is the vectorization (applied column-wise) of matrix $M$.
Then we introduce an additional variable $y$ and constraint $y^2 = 1$ to define
\begin{equation}
\begin{aligned}
x^{AB}_{j_1j_2} &\doteq  [\ r_s^T,\ r_p^T,\ y,\ D_{j_1},\ D_{j_2}\ ]^T \in \mathbb{R}^{(9+27+1+2) \times 1}.
\end{aligned}
\end{equation}
Then the loop error of the edge $\{j_1,j_2\}$ is rewritten as 
\begin{equation}
\begin{aligned}
\label{equ:error}
e^{AB}_{j_1j_2} &=[\ \widehat{t}_{B_{j_1}B_{j_2}}^T \otimes I,\ \text{vec}(\widehat{R}_{B_{j_1}B_{j_2}})^T \otimes I, \\ 
&-\widehat{t}_{A_{j_1}A_{j_2}},\ R_{A_1A_{j_1}} b_{j_1},\ -R_{A_1A_{j_2}} b_{j_2}\ ] x^{AB}_{j_1j_2},
\end{aligned}
\end{equation} The derivation of Equ.(\ref{equ:error}) from Equ.(\ref{equ:constraint}) is given in supplementary material.

\subsection{Mutual Localization with Anonymous Measurements}
\label{subsec:mixed-binary}

In this section, we extend the above loop error to the case with $N$ observed robots.  When the amount of observed robot increases to $N\geq2$, the correct correspondence of bearing measurement sequence $b^X=\{b^X_j\}_{j\in J}$ and estimated pose trajectory $T_Y=\{R_{Y_j}, t_{Y_j}\}_{j\in J}$ is hard to provide. Here $X,Y$ are indexes of measurement sequence and estimated trajectory respectively. Recovering the correspondence of a set of measurement sequences and a set of trajectories is called anonymity recovery problem. To solve it, we introduce binary variables $\Theta = \{\theta_{XY}\}_{X,Y \in [1,N]}$ , in which the $\emph{binary constraint}$ ($\theta_{XY} = \{0,1\}$) indicates whether the $X^{\text{th}}$ bearing measurement corresponds to the $Y^{\text{th}}$ trajectory ($\theta_{XY}$ = 1) or not ($\theta_{XY}$ = 0). And the $\emph{correspondece constraints}$ can be written as
\begin{equation}
\begin{aligned}
	\sum_{X} \theta_{XY} = 1 ,\sum_{Y} \theta_{XY} = 1, \forall X,Y\in[1,N].
\end{aligned}
\end{equation}
The above constraints are to guarantee that the measurement sequences and the estimated trajectories have one-to-one correspondence.

We use the binary variables to rewrite the loop error Equ.(\ref{equ:constraint}) for the $X^{\text{th}}$ measurement as follow
\begin{equation}
\begin{aligned}
\label{equ:multi_constraint}
e^{X}_{j_1j_2} = R_{A_1A_{j_2}} b^X_{j_2} D^X_{j_2} - R_{A_1A_{j_1}} b^X_{j_1} D^X_{j_1} + \widehat{t}_{A_{j_1}A_{j_2}} - \\
\sum_{Y=1}^N \theta_{XY}(R_{AY} \widehat{R}_{Y_{j_1}Y_{j_2}} {^Y\bar{P}} + s_{AY}R_{AY} \widehat{t}_{Y_{j_1}Y_{j_2}})    .
\end{aligned}
\end{equation}

Now we convert the mixed-integer expression to a linear form. Firstly, we denote the parameters that need to be estimated for robot $Y$ as $^Y\mathbb{P} \doteq [s_{AY}, ^Y\bar{P}^T]^T$. Then we define extra variables $^Y\mathbb{P}_X \doteq \theta_{XY} {^Y\mathbb{P}}$. Furthermore, we define the following variables,
\begin{gather}
r_{XY} \doteq \text{vec}(^Y\mathbb{P}_X^T \otimes R_{AY}) \in \mathbb{R}^{36 \times 1},\\
r_{X} \doteq \text{vstack}(\{r_{XY}\}_{Y=1}^N) \in \mathbb{R}^{36 N \times 1}, \\
D_{X} \doteq \text{vstack}(\{D^X_j\}_{j\in J}) \in \mathbb{R}^{n \times 1}, \\
r \doteq \text{vstack}(\{r_X\}_{X=1}^N) \in \mathbb{R}^{36 N^2 \times 1},\\
D \doteq \text{vstack}(\{D_X\}_{X=1}^N) \in \mathbb{R}^{n N \times 1},\\
x \doteq [\ r^T,\ y,\ D^T\ ]^T \in \mathbb{R}^{(36 N^2 + 1 + n N) \times 1}.
\end{gather}
 where the notation $\text{vstack}(G)$ stacks all variable in $G$ vertically and $n$ is the number of measurements. We use variable $x$ to rewrite Euq. (\ref{equ:multi_constraint}) in linear form as $e^{X}_{j_1j_2} = c^{X}_{j_1j_2} x,X\in[1,N]$. Detailed formulation of $c^{X}_{j_1j_2}$ is given in supplementary material. Then the error of each measurement sequence is used to formulate a nonconvex least-square problem
 \begin{figure}
 	\centering
 	\includegraphics[width=0.4\textwidth]{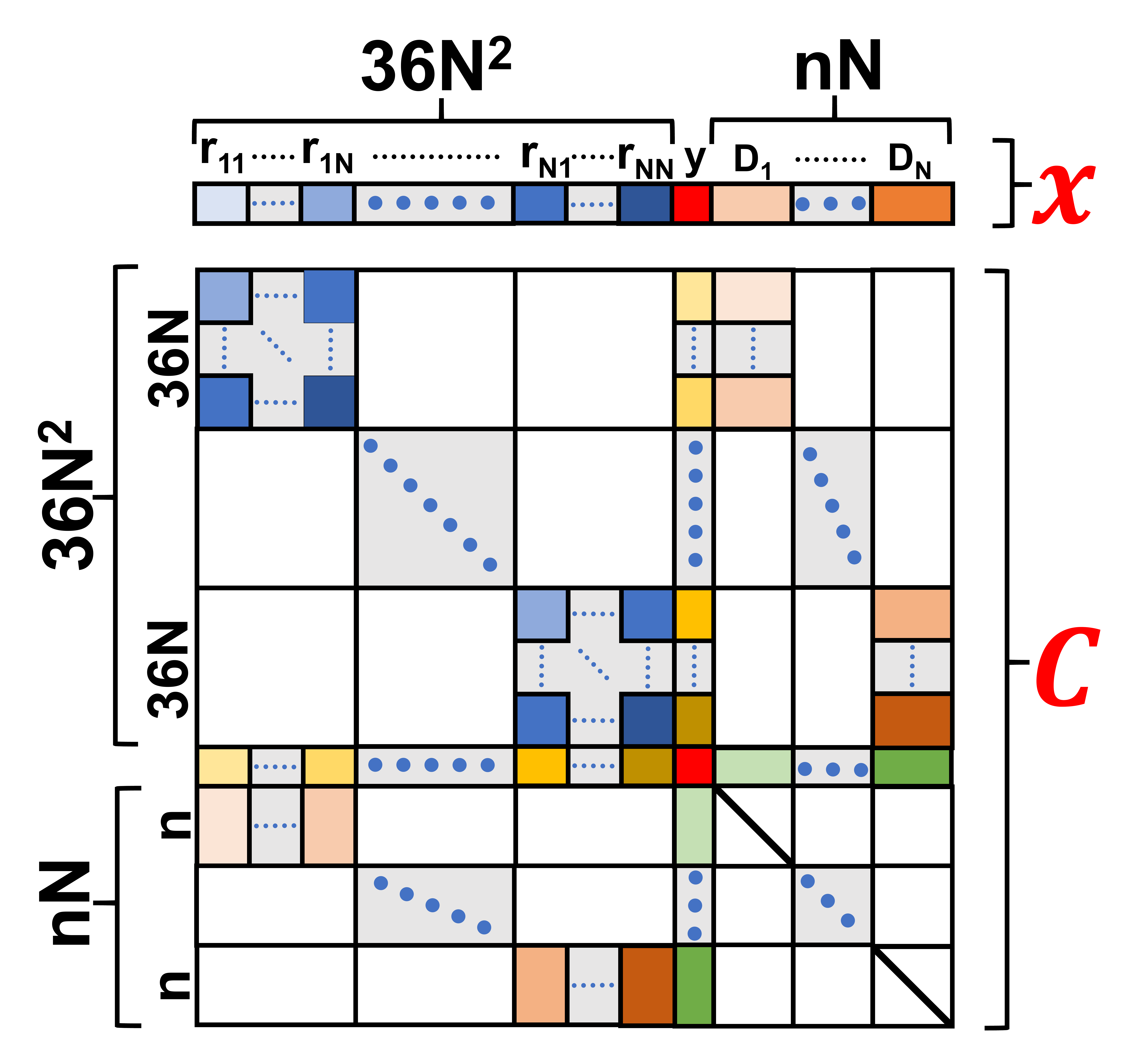}
 	\caption{\label{fig:Matrix} Structure of decision variable $x$ and cost matrix $\text{C}$ in Problem \ref{pro:origin}.}
 	\vspace{-0.7cm}
 \end{figure}
\begin{problem}[Original Problem]
\label{pro:origin}
\begin{equation}
\begin{aligned}
x^* =&\underset{x}{\arg\min}\ \sum_{X\in[1,N] \atop \{j_1,j_2\}\in J}(e^{X}_{j_1j_2})^T w^X_{j_1j_2} e^{X}_{j_1j_2}\\
=&\underset{x}{\arg\min}\ \sum_{X\in[1,N] \atop \{j_1,j_2\}\in J}x^T (c^{X}_{j_1j_2})^T w^X_{j_1j_2} c^{X}_{j_1j_2} x\\
=&\underset{x}{\arg\min}\ x^T \underbrace{(\sum_{X\in[1,N] \atop \{j_1,j_2\}\in J} (c^{X}_{j_1j_2})^T w^X_{j_1j_2} c^{X}_{j_1j_2})}_{:=C} x\\
s.t.&  D^X_j >0,\ s_{AY} > 0,\ R_{AY} \in SO(3),\ \\
&\sum_{X} \theta_{XY} = 1 ,\sum_{Y} \theta_{XY} = 1, \theta_{XY} \in \{0,1\},
\end{aligned}
\end{equation}
\end{problem}
where $w^X_{j_1j_2}$ is the measurement confidence parameter. The structure of $x$ and $C$ is shown in Fig.\ref{fig:Matrix}. Note that $C$ is a Gram matrix, so it is positive semidefinite and symmetric.

\subsection{Marginalization and Auxiliary Variables}
\label{subsec:marginalization}

In this subsection, we following the procedures in \cite{giamou2019certifiably} to marginalize the distance variables using Schur Complement. We write cost matrix $C$ as
\begin{gather}
C = \begin{bmatrix} C_{\bar{\mathcal{D}},\bar{\mathcal{D}}} & C_{\bar{\mathcal{D}},\mathcal{D}} \\
C_{\mathcal{D},\bar{\mathcal{D}}} & C_{\mathcal{D},\mathcal{D}} \end{bmatrix},
\end{gather}
where the subindex $\mathcal{D}$ stands for the set of indexes corresponding to the distance variables or not (subindex $\bar{\mathcal{D}}$). Then we eliminate distance variables $D$  and obtain
\begin{problem}[Marginalized Problem]
\label{pro:marg}
\begin{gather}
z^* =\underset{z}{\arg\min}\ z^T \bar{C} z \notag \\
s.t. \quad s_{AX} > 0, R_{AY} \in SO(3),\ \\
\sum_{X} \theta_{XY} = 1 ,\sum_{Y} \theta_{XY} = 1, \theta_{XY} \in \{0,1\}, \notag
\end{gather}
\end{problem}
where $
z = [r^T,\ y]^T$ and $\bar{C} = C / C_{\mathcal{D},\mathcal{D}} = C_{\bar{\mathcal{D}},\bar{\mathcal{D}}} - C_{\bar{\mathcal{D}},\mathcal{D}}C_{\mathcal{D},\mathcal{D}}^{-1}C_{\mathcal{D},\bar{\mathcal{D}}}$. Note that after marginalization, the number of involved variables is solely related to $N$. In contrast, exiting local optimization methods' computation is not only related to $N$ but also the number of measurements.

In our formulation, for each variable $r_{XY} = \text{vec}(\theta_{XY} {^Y\mathbb{P}} \otimes R_{AY})$, which involves $\theta_{XY}$ and $^Y\mathbb{P} = [\ s_{AY}, ^Y\bar{P}^T\ ]^T$, we generalize the $SO(3)$ constraints in \cite{briales2018certifiably} for our formulation as follow 
\begin{gather}
(\mu R_{AY})^T (\mu R_{AY}) = \mu^2 I, \label{e1}\\
(\mu R_{AY}) (\mu R_{AY})^T = \mu^2 I, \label{redundant1}\\
(\mu R_{AY})^{(i)} \times (\mu R_{AY})^{(j)} =\mu (\mu R_{AY})^{(k)} \label{redundant2}, \\
\ \ \ \ \ \ \ \ \ \ \ \ \ \ \ \ \ \forall (i,j,k) = {(1,2,3),(2,3,1),(3,1,2)}. \notag
\end{gather}
The variable $\mu$ is of the form $\mu = \theta_{XY} h_Y$, where $h_Y$ could be the term $s_{AY}, ^Y\mathbb{P}^{(1)}, ^Y\mathbb{P}^{(2)} \text{or}\ ^Y\mathbb{P}^{(3)}$. In actual, constraints (\ref{redundant1}) and (\ref{redundant2}) are redundant, and we will study the effectiveness of adding them in Sec. \ref{sec:experiments}.

However, it is still intractable to directly optimize the current problem. Since there is no direct variable corresponding to $\theta_{XY} h_Y$ in decision variable $z$, the above constraints can not be explicitly formulated into quadratic constraints in term of $z$. Similarly, since $z$ neither includes $\theta_{XY}$, the binary constraint $\theta_{XY} \in \{0,1\}$, which can be written as $\theta_{XY}^2 - \theta_{XY} = 0$, and the correspondence constraints all can not be constructed with decision variable $z$.

To address above issues, the key of next step is to introduce auxiliary variables, although they are not involved in cost function directly. According to the above analysis, we need to add auxiliary variables to represent $\theta_{XY} h_Y$ and $\theta_{XY}$. Besides, it is necessary to link the auxiliary variables for $\theta_{XY} h_Y$ and $\theta_{XY}$ to actual decision variable $z$ by adding variables representing $h_Y$ and $\text{vec}(h_Y R_{AY})$ and equality relationship constraints
\begin{gather}
\underline{\theta_{XY}h_Y}=\underline{\theta_{XY}}\ \underline{h_{Y}}, \\
\underline{\text{vec}(\theta_{XY}h_Y R_{AY})} = \underline{\theta_{XY}}\ \underline{ \text{vec}(h_Y R_{AY})}.
\end{gather}
where the underlines denote independent variables.

Summarize all necessary auxiliary variables as follows
\begin{enumerate}
	\item \textbf{Lifted Rotation Variable: $\ell$}
	\begin{gather}
	\ell_Y \doteq vec(^Y\mathbb{P}^T \otimes R_{AY}) \in \mathbb{R}^{36 \times 1},\\
	\ell \doteq \text{vstack}(\{\ell_Y\}_{Y=1}^N) \in \mathbb{R}^{36N \times 1}.
	\end{gather}
	\item \textbf{Binary Variable: $\varphi_\theta$}
	\begin{gather}
	\varphi_\theta^X \doteq \text{vstack}(\{\theta_{XY}\}_{Y=1}^N) \in \mathbb{R}^{N \times 1} ,\\
	\varphi_\theta \doteq \text{vstack}(\{\varphi_\theta^X\}_{X=1}^N) \in \mathbb{R}^{N^2 \times 1}.
	\end{gather}
	\item \textbf{Scale Ratio and Inner Bias Variable: $\varphi_h$}
	\begin{gather}
	\varphi_h \doteq \text{vstack}(\{^Y\mathbb{P}\}_{X=1}^N) \in \mathbb{R}^{4N \times 1}.
	\end{gather}
	\item \textbf{Lifted Scale Ratio and Inner Bias Variable: $\varphi_\mu$}
	\begin{gather}
	\varphi_\mu^X \doteq \text{vstack}(\{^Y\mathbb{P}_X\}_{Y=1}^N) \in \mathbb{R}^{4N \times 1}, \\
	\varphi_\mu \doteq \text{vstack}(\{\varphi_\mu^X\}_{X=1}^N) \in \mathbb{R}^{4N^2 \times 1}.
	\end{gather}
\end{enumerate}

Now we define the final decision variable 
$$
\bar{z} \doteq [z^T, \ell^T, \varphi_\theta^T, \varphi_h^T, \varphi_\mu^T]^T,
$$
and use it to formulate all constraints in quadratic terms $\bar{z}^T Q_i \bar{z} = g_i, i \in [1,m]$, where $m$ is the number of constraints. For detailed derivation of $Q_i$, we refer readers to supplementary material. Now we obtain
\begin{problem}[QCQP Probem]
\label{pro:qcqp}
\begin{gather}
f* =\underset{\bar{z}}{\min}\ \bar{z}^T Q_0 \bar{z} \notag \\
s.t. \quad \bar{z}^T Q_i \bar{z} = g_i, i=1,..,m,
\end{gather}
where $Q_0 = \begin{bmatrix} \bar{C} & 0_{d_z\times d_a} \\
0_{d_a\times d_z} & 0_{d_a\times d_a}\end{bmatrix}$. $d_z$ and $d_a$ are dimensions of $z$ and auxiliary variables respectively. 
\end{problem}

However, the formulated non-convex QCQP is still nontrivial to solve. In next section, we provide a complete algorithm using SDR to get the global optimal solution of Problem \ref{pro:qcqp}.

\section{Certifiably Global Optimization by Semidefinite Relaxation}
\label{sec:optimization problem}
In this section, we will firstly apply semidefinite relaxation to Problem \ref{pro:qcqp} in Sec.\ref{subsec:sdp}. Then we recover data correspondence and relative poses from the solution of the SDP problem in Sec.\ref{subsec:recovery}. Lastly, we provide a condition under which the zero-duality-gap and one-rank-solution can be guaranteed in noise-free cases in Sec.\ref{subsec:tightness}. 

\subsection{Semidefinite Relaxation and Dual Problem}
\label{subsec:sdp}
As stated above, Problem \ref{pro:qcqp} is non-convex. Fortunately, it can be relaxed to a convex SDP, known as Shor's relaxation. By introducing matrix variable $Z \doteq \bar{z}\bar{z}^T$, we have
\begin{gather}
\label{equ:trace}
\bar{z}^TQ_i\bar{z} = \text{tr}\,(\bar{z}^TQ_i\bar{z}) = \text{tr}\,(Q_i\bar{z}\bar{z}^T) = \text{tr}\,(Q_iZ),
\end{gather}
where $\text{tr}(M)$ is the trace of matrix $M$. Together with Equ.(\ref{equ:trace}) and dropping the constraint of $\text{rank}\,(Z)=1$, we obtain the following problem.
\begin{problem}[Primal SDP]
\label{pro:primal sdp}
\begin{gather}
f^*_{\text{primal}} =\underset{Z}{\min}\ tr(Q_0 Z)  \notag\\
s.t.  Z \succeq 0, tr(Q_iZ) = g_i, i=1,...,m, 
\end{gather}
\end{problem}
which is convex and can be solved by off-shelf solvers using primal-dual interior point method. Its dual problem is
\begin{problem}[Dual SDP]
\label{pro:dual sdp}
\begin{gather}
f^*_{\text{dual}} =\underset{\lambda}{\max}\ g^T\lambda \notag \\
s.t. Q(\lambda) = Q_0 - \sum_{i}\lambda_iQ_i \succeq 0, i=1,...,m,
\end{gather}
\end{problem}
where $g=[g_1,...,g_m]^T$, $\lambda=[\lambda_1,...,\lambda_m]^T$. 

Once $Z^*$, the solution of Problem \ref{pro:primal sdp}, is obtained, we denote the part of $Z^*$ that corresponds to variable $z$ as $\mathcal{Z}^* \doteq Z^*_{[1:36N^2, 1:36N^2]}$. Moreover, if zero-duality-gap  ($f^*_{\text{primal}} = f^*_{\text{dual}}$) and one-rank-solution ($\text{rank}\, (\mathcal{Z}^*)=1$) hold, we can obtain the global optimal solution $z^*$ of Problem \ref{pro:marg} as described in Sec. \ref{subsec:recovery}. Actually, both the above conditions are satisfied in noise-free cases, which is proved in Sec. \ref{subsec:tightness}.

\subsection{Recovery from the tight SDP solution}
\label{subsec:recovery}
Given $\mathcal{Z}^*$, we need to recover the optimal correspondences and relative poses. According to the one-rank-solution ($\text{rank} (\mathcal{Z}^*)=1$), we firstly deploy a rank-one decomposition to obtain $z^*\in \mathbb{R}^{36N^2 \times 1}$. Denoting $r_{XY}^* \in \mathbb{R}^{36 \times 1}$ as slices of $z^*$ corresponding to variable $r_{XY}$, we define $M_{AY}^X \doteq \theta_{XY}^{*} {^Y\mathbb{P}^{*T}} \otimes R_{AY}^* = \text{mat}(r^*_{XY}, [12, 3])$, where $\text{mat}(v,[r,c])$ means reshape the vector $v$ to one $r \times c$ matrix by col-first order. Note that $M_{AY}^X$ is either zero matrix or non-zero matrix due to the binary variable $\theta_{XY}$. So we set $\epsilon = 10^{-5}$ and  take  $||M_{AY}^{X}||_2 > \epsilon$ to indicate that the $X^{\text{th}}$ measurement corresponds to the $Y^{\text{th}}$ estimated trajectory.

Then for each $M_{AY}^X$ whose corresponding $\theta_{XY} > \epsilon$, we recover the scale ratio and relative rotation
\begin{gather}
\mathbb{S}_{AY}^* := s_{AY}^* R_{AY}^* = M^X_{AY[1:3,1:3]}, \\
s_{AY}^* = \sqrt[3]{\text{det}(\mathbb{S}_{AY}^*)}, R_{AY}^* = \mathbb{S}_{AY}^*/ s_{AY}^*, 
\end{gather}
and inner bias $^YP^*$ similarly.

Recall that we have marginalized the distance variable $D$ in Sec.\ref{subsec:marginalization}, we now recover the optimal $D^*$ as 
\begin{gather}
D^*(r^*) = -C_{\mathcal{D},\mathcal{D}}^{-1}C_{\mathcal{D},\bar{\mathcal{D}}}r^*.
\end{gather}

Furthermore, the optimal relative translation $t_{AY}^*$ is recovered using $D^*$ as   follow
\begin{gather}
t_{AY}^* = \sum_{j\in J}(  t_{A_1A_j} + R_{A_1A_j}(D_j^{Y*}\ b_j^Y) - \\
s_{AY}^* R_{AY}^* (R_{Y_1Y_j} {^YP^*} + t_{Y_1Y_j})). \notag
\end{gather}

\subsection{Tightness of Semidefinite Relaxation}
\label{subsec:tightness}
In this subsection, we aim to prove that there are zero-duality-gap and one-rank-solution in noise-free cases. Firstly, we introduce a lemma and a corank-one condition.

\begin{lemma}
\label{lemma:psd}
If $C \in \mathbb{R}^{n\times n}$ be positive semidefinite and $x^TCx=0$ for a vector $x$, then $Cx=0$.
\end{lemma}

\begin{definition}
	For Problem \ref{pro:origin}, the \emph{corank-one condition} holds if  the number of \emph{independent} measurements $n$ and the number of observed robots $N$, satisfy that $n \geq 18N+2$, where \emph{independent} measurements mean that $\{c^{X}_{j_1j_2}\}$ are linearly independent vectors.
\end{definition}
Based in this condition, we have 
\begin{lemma}
\label{lemma:corank} Assume that the corank-one condition holds, then the cost matrix $C$ is semidefinite and has corank one in noise-free cases. Furthermore, after Schur Compliment, $\bar{C}$ is also semidefinite and has corank one.  
\end{lemma}

The detailed proof of above two lemmas can be seen in supplementary material. Then we apply Lemma 2.1 in \cite{cifuentes2021local} to our problem and introduce the following proposition.
\begin{proposition}
If bearing measurements are noise-free, their is zero-duality-gap between Problem \ref{pro:qcqp} and Problem \ref{pro:dual sdp}. Furthermore, once the corank-one condition is satisfied and given the solution $Z^*$ of Problem \ref{pro:primal sdp}, we have $\text{rank} (\mathcal{Z}^*) = 1$, and its rank-one decomposition $z^*$ is the global optimal minimum of Problem \ref{pro:qcqp}.
\end{proposition}
\begin{proof}
Let $\tilde{\bar{z}} = [\tilde{z}^T, \tilde{\ell}^T, \tilde{\varphi_\theta}^T, \tilde{\varphi_p}^T, \tilde{\varphi_\mu}^T]^T $be a feasible point in Problem \ref{pro:qcqp} where $\tilde{z}, \tilde{\ell}, \tilde{\varphi_\theta}, \tilde{\varphi_p}, \tilde{\varphi_\mu}$ are all ground truth. Let $\tilde{\lambda}=0$ be a feasible point in Problem \ref{pro:dual sdp}. Then the zero-duality-gap is guaranteed since the below three conditions needed in Lemma 2.1 in \cite{cifuentes2021local} are satisfied:
(i) Primal feasibility. In noise-free cases, the ground truth certainly satisfy constraints in Problem \ref{pro:qcqp}. 
(ii) Dual feasibility. $Q(\tilde{\lambda}) = Q_0 - \sum_{i=1}^m\tilde{\lambda}_iQ_i = Q_0 = \begin{bmatrix} \bar{C} & 0_{d_z \times d_a} \\
0_{d_a \times d_z} & 0_{d_a \times d_a}\end{bmatrix} \succeq 0$. 
(iii) Lagrangian multiplier. Since $\tilde{\bar{z}}$ is ground truth,  $\tilde{\bar{z}}^TQ_0\tilde{\bar{z}}$ equals to the optimal cost in Problem \ref{pro:origin}, which equals to 0. Recall that $Q_0$ is semidefinite according to Lemma. \ref{lemma:corank}, so $Q(\tilde{\lambda}) \tilde{\bar{z}} = 0$ is obtained based on Lemma. \ref{lemma:psd}. 

Furthermore, suppose $Z^*$ is the optimal solution of Problem \ref{pro:primal sdp}. Then $ Z^* \neq 0$ since at least one $g_i \neq 0$.  By complementary
slackness, $\text{tr}(Q(\tilde{\lambda}) Z^*) = \text{tr}(\begin{bmatrix} \bar{C} \mathcal{Z}^* & 0_{d_z \times d_a} \\
0_{d_a \times d_z} & 0_{d_a \times d_a}\end{bmatrix}) = 0$, so $\text{tr}(\bar{C} \mathcal{Z}^*) = 0$. And since $\bar{C}$ and $\mathcal{Z}^*$ are both positive semidefinite, $\text{rank} (\bar{C}) + \text{rank} (\mathcal{Z}^*) <= N$. So, if $\text{corank} (\bar{C}) = 1$, $\text{rank} (\mathcal{Z})^* = 1$. Moreover, its rank-one decomposition $\tilde{z}$ is the unique optimum of Problem \ref{pro:qcqp}.
\end{proof}





\section{Experiments}

In this section, we firstly confirm the optimality and efficiency of our method by comparing it against the alternating minimization (AM)\cite{jang2021multirobot} and the Levenberg-Marquardt (LM) methods. 
Next, to present the robustness of our method, we compare its performances under different levels of noise. 
Then, we show the results of our method with different robot number and noise. 
Finally, we apply our algorithm in real-world, using estimated odometry from different sources. 

\label{sec:experiments}
\subsection{Experiments on Synthetic Data}
\label{subsec:benchmark}
To simulate bearing measurement, we generate random trajectories for multiple robots. An example simulated environment is shown in Fig. \ref{fig:simulation}. Robots trace circular routes around different centers over a common landscape consisting of multiple random sinusoidal functions. All trajectories have the same length. Then, for observer robot $A$ and observed robot $Y$, we use their global poses to generate noisy bearing measurement as follow
\begin{gather}
b_j^Y = R_{A_j}^{-1}(t_{Y_j} + \mathcal{N}(1,\sigma) R_{Y_j} {^YP} - t_{A_j}).
\end{gather} 
where $\mathcal{N}(1,\sigma)$ is Gaussian distribution with standard deviation $\sigma$ . Then, we take the first pose of each trajectory as the local world frame and obtain each robot's local poses, which will be shared with other robots for estimation.
 \subsubsection{Optimality and Runtime}
Given a certain initial value of relative rotation matrix $R$, both AM and LM can converge to a local minimum, with error distributions shown in Fig. \ref{fig:init}. 
In this figure, each cell denotes the $\text{L}^2$-norm error of the  estimated relative pose. 
Fig. \ref{fig:init} states, the optimization converges to local minimums if the distance between the initial values and ground truth is large. 

Then we compare our method with these two methods in optimality and efficiency for four problems: RPE without scale ratio and inner bias (RPE-only), RPE with scale ratio (RPE-S), RPE with inner bias (RPE-B), and RPE with scale ratio and inner bias (RPE-SB). 
For each problem, we conduct 1000 experiments using different measurements.
The left figure of Fig. \ref{fig:benchmark_all} shows that for all problems, our method can always obtain the optimal solution, while both AM and LM are trapped in local minimums with random initial values. 
For efficiency, since our formulation fixes the number of variables by marginalizing the distance variables $D$ (see Sec.\ref{subsec:marginalization}), its computing time is only related to the number of observed robots. 
In contrast, the number of variables in local optimization methods AM and LM increase with measurement number. 
The right figure of Fig. \ref{fig:benchmark_all}, which presents the mean runtime with 200 bearing measurements, show that our method solves all problems faster.

\begin{figure}[t]
	\centering
	\includegraphics[width=0.5\textwidth]{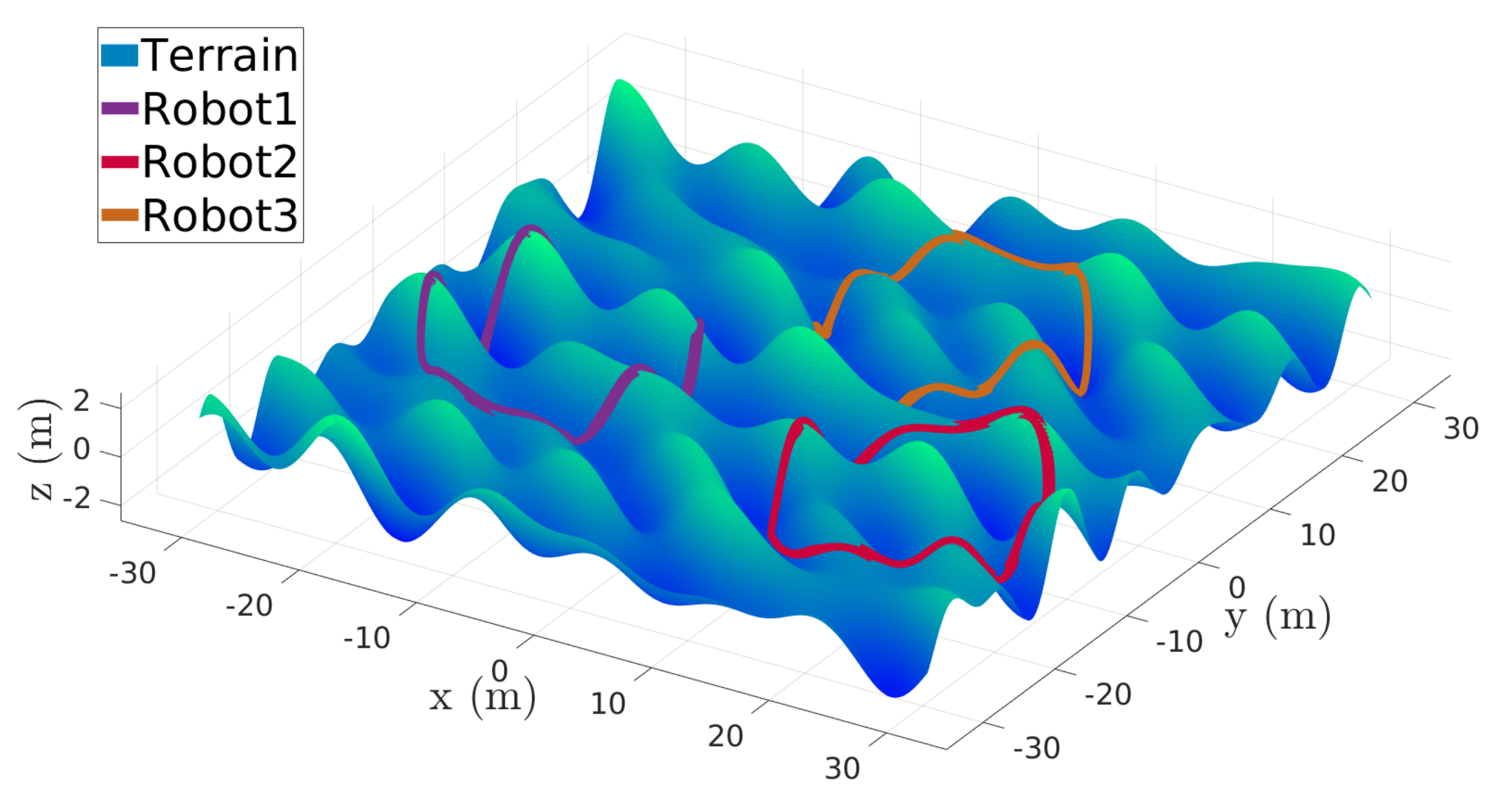}
	\caption{\label{fig:simulation} Three random trajectories in a simulated environment.}
	\vspace{-2.8cm}
\end{figure}

\begin{figure}[b]
	\centering
	\includegraphics[width=0.5\textwidth]{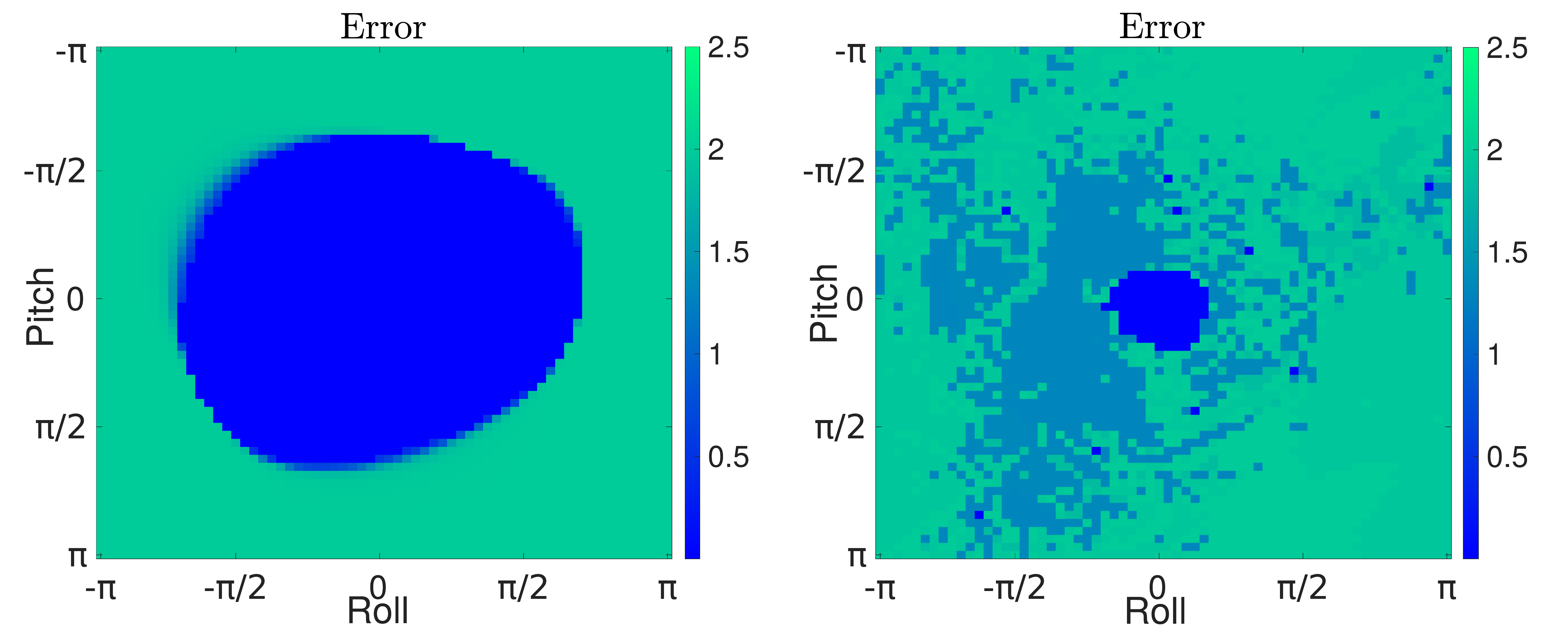}
	\caption{\label{fig:init} Error heatmaps for two local optimization methods over a uniform sampled initial value $R$, which is generated by rotating ground truth with roll (x-axis) and pitch (y-axis) from $-\pi$ and $\pi$. Green region denotes range of initial values in which these approach will drop into local minimums.}
\end{figure}

 \begin{figure}[t]
	\centering
	\includegraphics[width=0.5\textwidth]{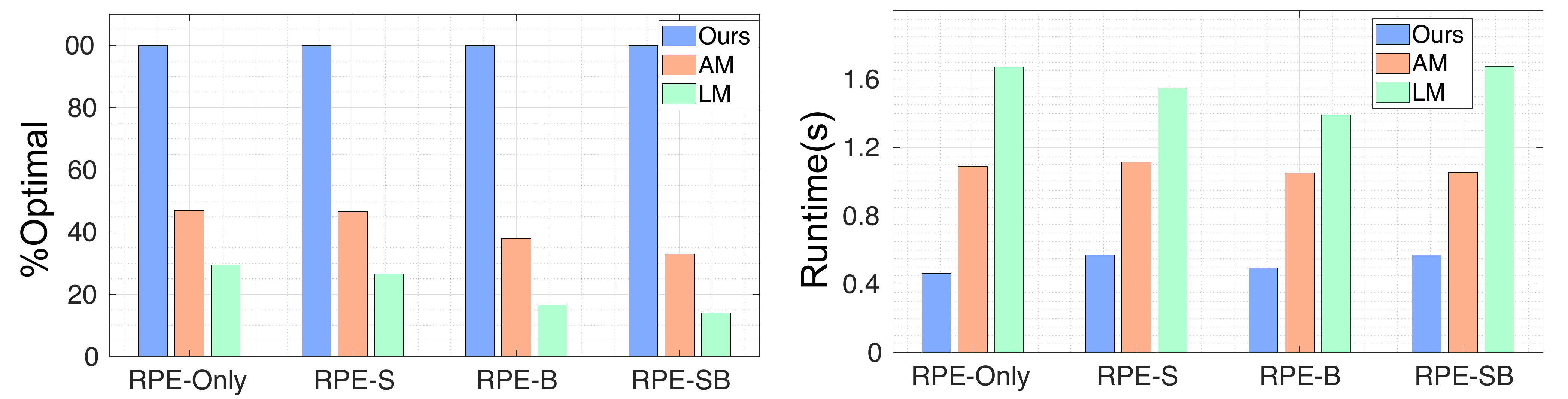}
	\caption{\label{fig:benchmark_all} Benchmark results between our method and local optimization algorithms for different problems.}
	\vspace{-1.7cm}
\end{figure}

 \begin{figure}[b]
 	\centering
 	\includegraphics[width=0.5\textwidth]{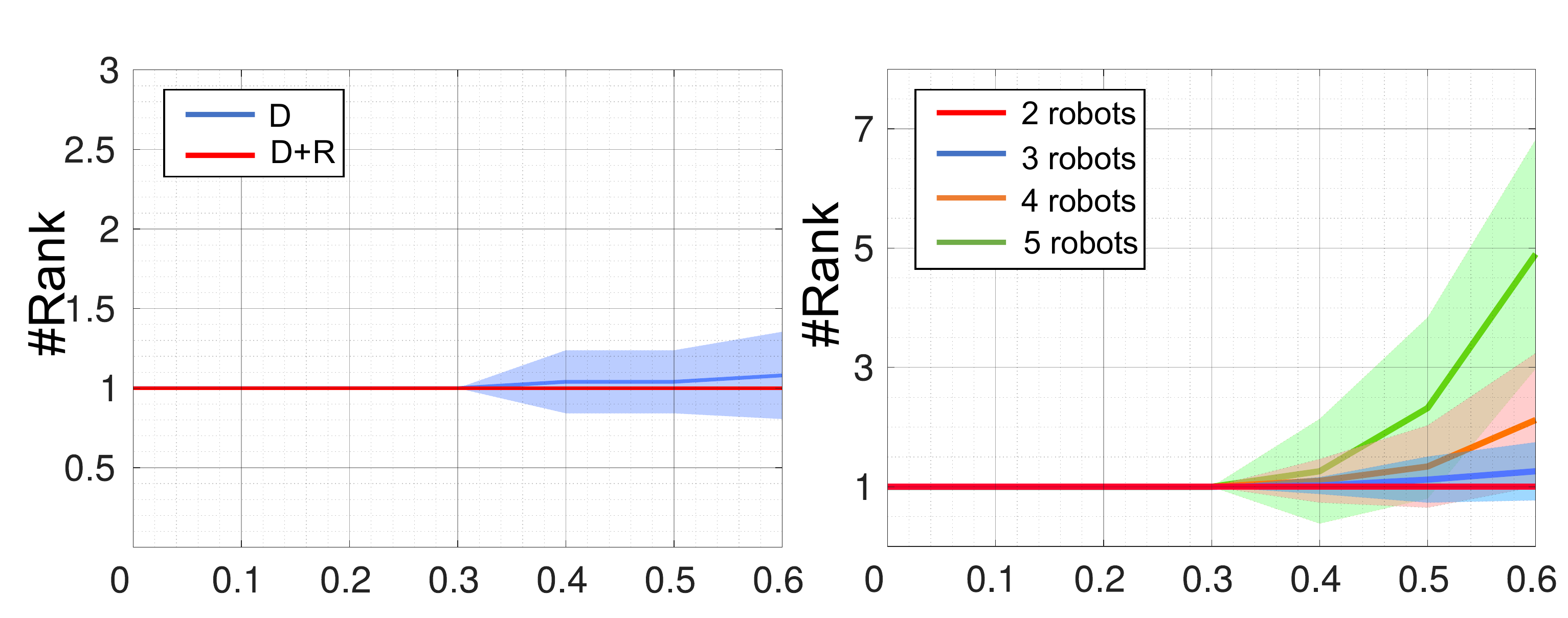}
 	\caption{\label{fig:all_rank}(Left) Comparison between method w/ and w/o redundant constraints (Right) Comparison with different number of robots. (solid line: mean; shaded area: 1-sigma standard deviation).}
 \end{figure}
 \begin{figure*}[t]
	\centering
	\includegraphics[width=1.0\textwidth]{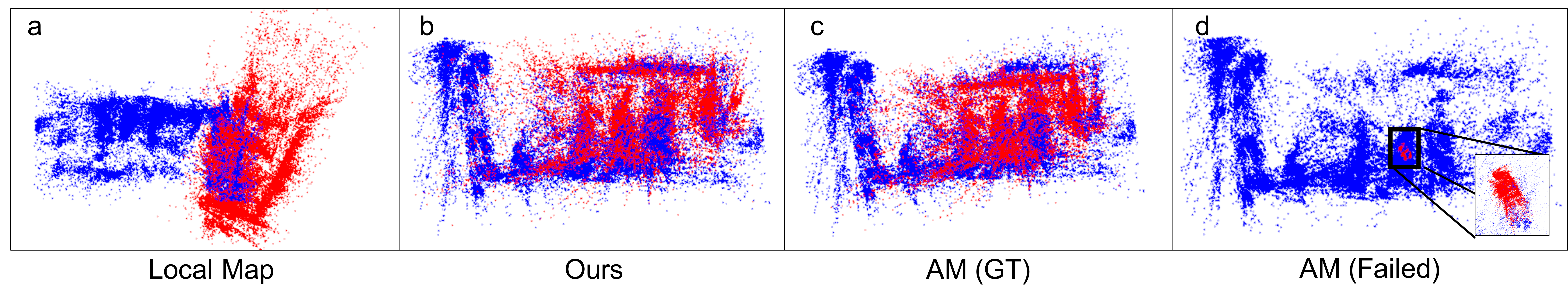}
	\caption{\label{fig:ORB} Map fusion results using feature maps from two robots, which launch at different place and observe each other when they rendezvous.}
	\vspace{-0.4cm}
\end{figure*}

 \subsubsection{Robustness}
To evaluate the robustness of our method and the effectiveness of the redundant constraints, we add different levels of noise into simulated measurements. 
We compare two versions of our method, the default version (D) and the augmented version which is added redundant rotation constraint (D+R). 
As the left plot of Fig. \ref{fig:all_rank} shows, for each noise level, the augmented version (D+R) recover an exact minimizer of the primal problem. However, the default version (D) does not obtain the one-rank solution under extreme noise ($\sigma \geq 0.4$).
 
Fig. \ref{fig:robust_benchmark} presents the performance of our method (D and D+R) and several local optimization methods. 
AM and LM utilize random rotation as the initial value, and AM (GT) and LM (GT) use the ground truth instead. 
Each figure represents 100 random trials on simulated data with different noise levels $\sigma$.
As Fig. \ref{fig:robust_benchmark} shows, our method is consistently more accurate compared to AM and LM and has comparable performance with AM (GT). Besides, we observe that under extreme noise ($\sigma$ = 0.5), our method still performs accurately.

\begin{figure}[t]
	\centering
	\includegraphics[width=0.5\textwidth]{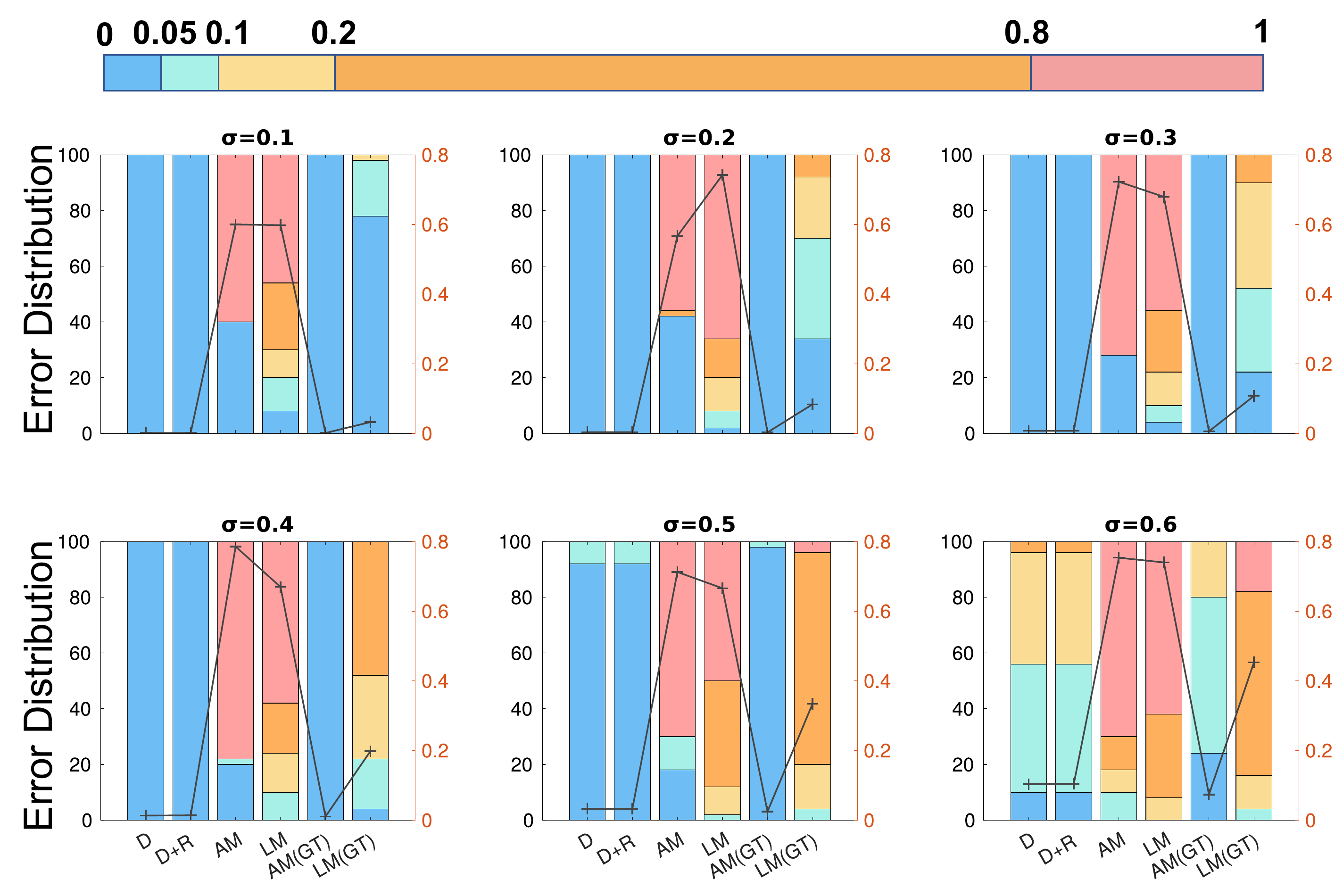}
	\caption{\label{fig:robust_benchmark} Comparison of error distribution between different methods. The top colorbar presents colors corresponding to different error range. In each subfigure, each bar denotes the percentage of error range and the black line represents the mean error.}
	\vspace{-0.7cm}
\end{figure}

\begin{figure}[b]
	\centering
	\includegraphics[width=0.5\textwidth]{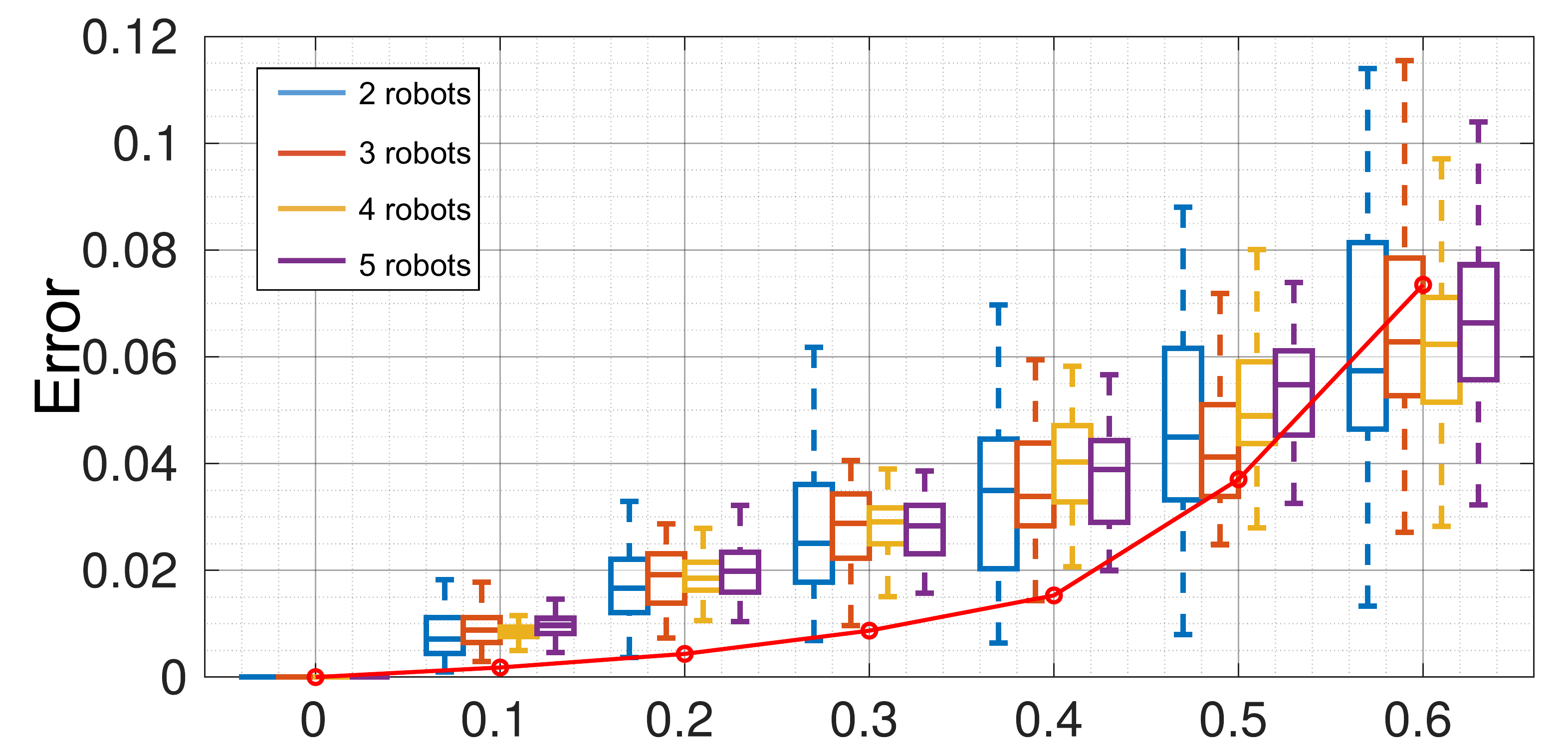}
	\caption{\label{fig:multi_error} Comparison of error distribution of our method with different number of robots under different levels of noise.}
\end{figure}

Furthermore, we conduct experiments with multiple observed robots under noise. 
The right figure of Fig.\ref{fig:all_rank} presents the trend of rank($\mathcal{Z}$) when the noise level increases. 
Although as the number of robots increases, the zero-duality-gap is easily influenced by noise, our method always obtains one-rank solution with a common noise level ($\sigma < 0.4$). 
Moreover, we present the error distribution of obtained solutions in Fig. \ref{fig:multi_error}. 
This figure states: (1) Increasing number of robots does not influence the estimation error majorly. (2) Although there is no one-rank-solution under extreme noise, the result of one-rank decomposition has comparable accuracy with AM (GT).

\subsubsection{Scalability} 
In our method, the number of variables is related to the squared number of observed robots.
Fig. \ref{fig:multi_time} shows the runtime of our method with different numbers of robots. 
According to the result, our method has an acceptable runtime in real multi-robot applications when the robot number is limited.

 \begin{figure}[b]
 	\centering
 	\includegraphics[width=0.5\textwidth]{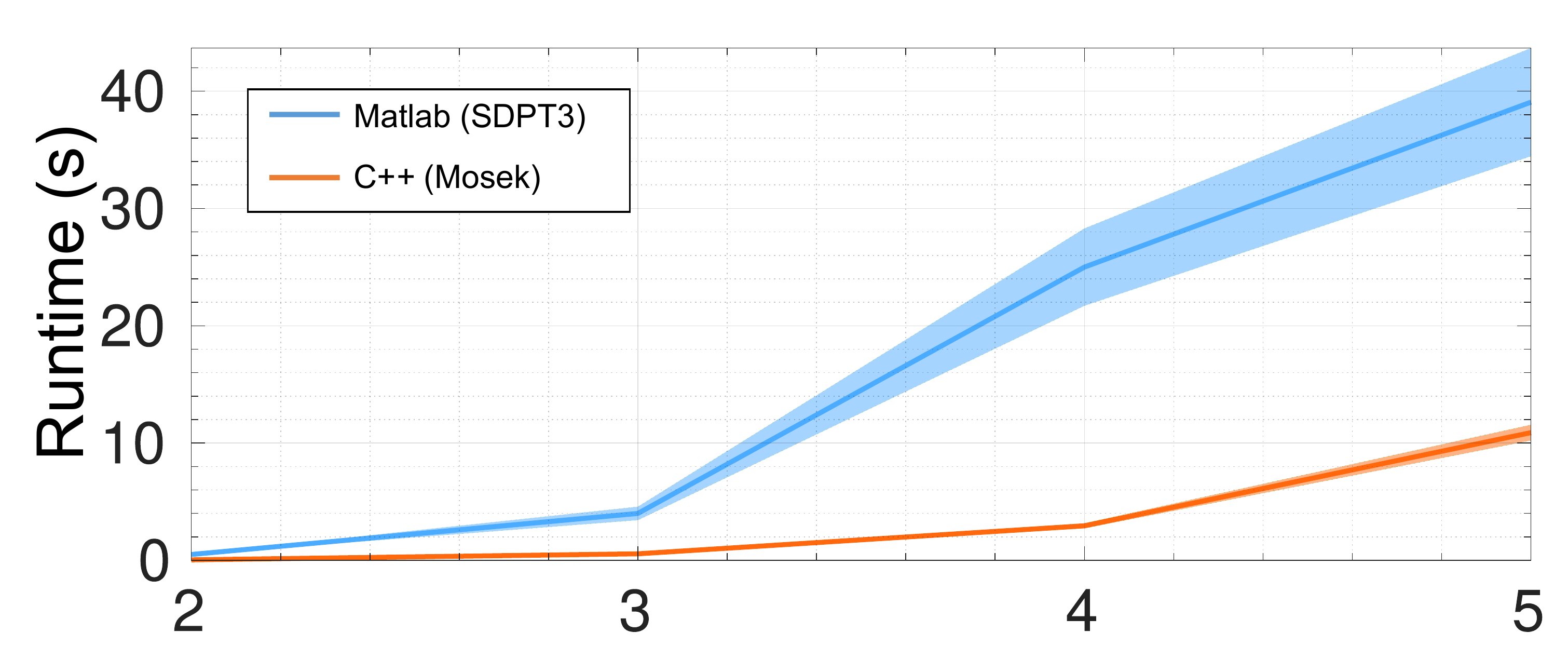}
 	\caption{\label{fig:multi_time} Runtime comparison using Matlab / C++ with different number of robots. (solid line: mean; shaded area: 1-sigma standard deviation)}
 \end{figure}
 

\subsection{Real-world Experiments}
\label{subsec:real-world}

\begin{figure}
	\centering
	\includegraphics[width=0.5\textwidth]{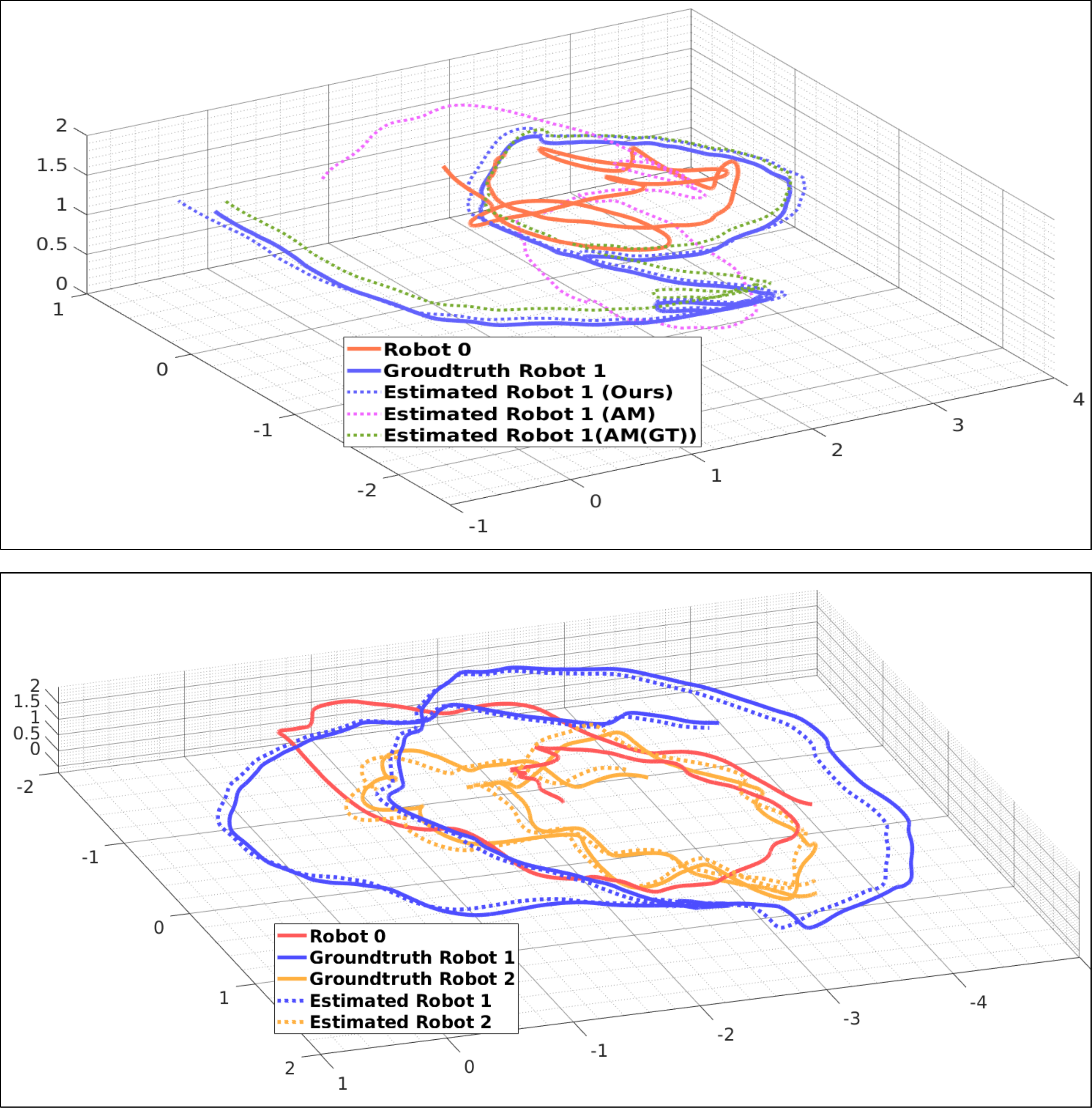}
	\caption{\label{fig:real_2} Estimated trajectories and ground truth in real-world experiments. Top: Two robots experiment with a observer (robot 0) and a observed robot (robot 1). Bottom: Three robots experiment with a observer robot (robot 0) and two observed robots (robot 1 and 2).}
\end{figure}


In real-world experiments, we use motion capture and VIO for odometry estimations and AprilTag for bearing measurements, with ground truth provided by vicon motion capture. Table \ref{tab:realworld} summarizes the results of comparison between our method and others. The experiments that use Vicon (300 measurements) is labeled "Vicon+AprilTag", and using VINS \cite{qin2017vins}(100 measurements) is labeled "VINS+AprilTag". 
Under each configuration, we conduct experiments with two and three robots. The ground truth correspondence comes from AprilTag. 
In experiments with three robots, AM (C+GT) optimizes with ground truth correspondence, while AM (w/o C+GT) does not. 
The results are in Fig.(\ref{fig:real_2}).

As Table \ref{tab:realworld} shows, AM and AM (w/o C+GT) all converge to small cost but obtain egregiously large $\text{L}^2$-norm error, which indicates that they are trapped in a local minimum. 
Compared with them, our globally optimal approach obtains the minimum cost in all experiments with most or secondly most small error. 
Note that, due to noise from odometry estimation and bearing measurements, obtaining the minimum cost does not mean obtaining the most accurate estimation. 
For runtime, the table shows that our algorithm has constant runtime which is independent to the number of measurements. 
It also indicates that our algorithm is suitable for bootstrapping other algorithms that use relative poses as initialization.

Finally, we apply our algorithm in multi-robot map fusion as Fig.\ref{fig:ORB} shows. 
In this experiment, each robot's local map comes from feature-based monocular SLAM. 
Compared with AM that traps into local minimum and fails to fuse maps, our result fuses robots' maps correctly without any initialization, while AM (GT) needs.

\begin{table}[]
	\label{tab:realworld}
	\caption{Real-world Experiments Results}
	\resizebox{1.01\columnwidth}{!}{
	\begin{tabular}{|c|c|c|c|cc|c|}
		\hline
		&                                     &                                   &                                        & \multicolumn{2}{c|}{\textbf{$\text{L}^2$ Error}}                                                                              &                                                                                   \\ \cline{5-6}
		\multirow{-2}{*}{\textbf{Scene}}                                             & \multirow{-2}{*}{\textbf{\#Robots}} & \multirow{-2}{*}{\textbf{Method}} & \multirow{-2}{*}{\textbf{Cost}}        & \multicolumn{1}{c|}{\begin{tabular}[c]{@{}c@{}}Trans. \\ (m)\end{tabular}} & Rot.                                  & \multirow{-2}{*}{\textbf{\begin{tabular}[c]{@{}c@{}}Runtime\\ (ms)\end{tabular}}} \\ \hline
		&                                     & Ours (D+R)                        & {\color[HTML]{FE0000} \textbf{0.0018}} & \multicolumn{1}{c|}{\textbf{0.24}}                                         & {\color[HTML]{330001} \textbf{0.063}} & \textbf{343.3}                                                                    \\ \cline{3-7} 
		&                                     & AM                                & 0.198                                  & \multicolumn{1}{c|}{2.52}                                                  & 2.83                                  & 1391.5                                                                            \\ \cline{3-7} 
		& \multirow{-3}{*}{2}                 & AM (GT)                            & {\color[HTML]{333333} 0.132}           & \multicolumn{1}{c|}{0.323}                                                 & 0.087                                 & 660.5                                                                             \\ \cline{2-7} 
		&                                     & Ours (D+R)                        & {\color[HTML]{FE0000} \textbf{0.0006}} & \multicolumn{1}{c|}{{\color[HTML]{000000} 0.092}}                          & {\color[HTML]{000000} 0.0688}         & \textbf{603.3}                                                                    \\ \cline{3-7} 
		&                                     & AM (C)                            & 0.727                                  & \multicolumn{1}{c|}{2.506}                                                 & 2.809                                 & 1289.1                                                                            \\ \cline{3-7} 
		&                                     & AM (C+GT)                         & {\color[HTML]{000000} 0.082}           & \multicolumn{1}{c|}{\textbf{0.0305}}                                       & \textbf{0.0305}                       & 632.7                                                                             \\ \cline{3-7} 
		\multirow{-7}{*}{\begin{tabular}[c]{@{}c@{}}VICON+\\ AprilTag\end{tabular}} & \multirow{-4}{*}{3}                 & AM (w/o C+GT)                      & 0.569                                  & \multicolumn{1}{c|}{1.53}                                                  & 0.290                                 & 901.5                                                                             \\ \hline
		&                                     & Ours (D+R)                        & {\color[HTML]{FE0000} \textbf{0.553}}  & \multicolumn{1}{c|}{\textbf{0.429}}                                        & \textbf{0.0716}                       & 180.1                                                                             \\ \cline{3-7} 
		&                                     & AM                                & 0.667                                  & \multicolumn{1}{c|}{2.597}                                                 & 2.823                                 & 312.3                                                                             \\ \cline{3-7} 
		& \multirow{-3}{*}{2}                 & AM (GT)                            & {\color[HTML]{000000} 0.650}           & \multicolumn{1}{c|}{{\color[HTML]{000000} 0.601}}                          & {\color[HTML]{000000} 0.159}          & \textbf{120.2}                                                                    \\ \cline{2-7} 
		&                                     & Ours (D+R)                        & {\color[HTML]{FE0000} \textbf{0.101}}  & \multicolumn{1}{c|}{{\color[HTML]{000000} 0.943}}                          & {\color[HTML]{000000} 0.205}          & 603.3                                                                             \\ \cline{3-7} 
		&                                     & AM (C)                            & 0.423                                  & \multicolumn{1}{c|}{1.62}                                                  & 2.82                                  & 512.3                                                                             \\ \cline{3-7} 
		&                                     & AM (C+GT)                         & {\color[HTML]{000000} 0.337}           & \multicolumn{1}{c|}{\textbf{0.773}}                                        & \textbf{0.118}                        & \textbf{131.1}                                                                    \\ \cline{3-7} 
		\multirow{-7}{*}{\begin{tabular}[c]{@{}c@{}}VINS+\\ AprilTag\end{tabular}}  & \multirow{-4}{*}{3}                 & AM (w/o C+GT)                      & 1.308                                  & \multicolumn{1}{c|}{7.24}                                                  & 2.82                                  & 305.6                                                                             \\ \hline
	\end{tabular}
}
\end{table}

\section{Conclusions and Future Work}
\label{sec:conclusion}
In this paper, we proposed a certifiably globally optimal algorithm for mutual localization problems with anonymous bearing measurements. With our method, we can obtain bearing-pose correspondences and relative poses between robots together. Furthermore, we provide a necessary condition for optimality guarantee and conduct extensive experiments to present the optimality and robustness compared with local optimization methods. In the future, we aim to explore the noise tolerance threshold of our method to provide a more powerful guarantee for application.

\bibliography{ICRA}
\end{document}